\documentclass{article}

\usepackage[colorlinks,linkcolor=red,anchorcolor=blue,citecolor=green]{hyperref}
\usepackage{arxiv}

\usepackage[utf8]{inputenc} % allow utf-8 input
\usepackage[T1]{fontenc}    % use 8-bit T1 fonts
\usepackage{url}            % simple URL typesetting
\usepackage{booktabs}       % professional-quality tables
\usepackage{amsfonts}       % blackboard math symbols
\usepackage{nicefrac}       % compact symbols for 1/2, etc.
\usepackage{microtype}      % microtypography
\usepackage{lipsum}		% Can be removed after putting your text content
\usepackage{graphicx}
\usepackage{natbib}
\usepackage{doi}
\usepackage{xcolor}         % colors
\usepackage{wrapfig}
\usepackage{enumitem}
\usepackage{amssymb}
\usepackage{amsmath}
\usepackage{amsthm}
\usepackage{algorithm}
\usepackage{algpseudocode}
\usepackage[normalem]{ulem}
\useunder{\uline}{\ul}{}
\usepackage{tcolorbox}
\usepackage{multirow}

\theoremstyle{plain}
\newtheorem{theorem}{Theorem}[section]

\theoremstyle{definition}
\newtheorem{definition}[theorem]{Definition}
\newtheorem{assumption}[theorem]{Assumption}
\theoremstyle{remark}

\title{Is Gradient Ascent Really Necessary?\\ Memorize to Forget for Machine Unlearning}

\author{
  Zhuo Huang$^{1}$,  % \thanks{Equal contributions.}
  Qizhou Wang$^{3}$,   % \footnotemark[1]
  Ziming Hong$^{1}$,
  Shanshan Ye$^{4}$,
  Bo Han$^{2,3}$,
  Tongliang Liu$^{1}$\\[1ex]
  \small{$^1$Sydney AI Centre, The University of Sydney;}
  \small{$^2$Hong Kong Baptist University;} \\
  \small{$^3$RIKEN AIP;}
  \small{$^4$University of Technology Sydney}
}

\date{}

\renewcommand{\shorttitle}{\textit{Is Gradient Ascent Really Necessary?}}

\begin{document}
\maketitle

\begin{abstract}
For ethical and safe AI, machine unlearning rises as a critical topic aiming to protect sensitive, private, and copyrighted knowledge from misuse. To achieve this goal, it is common to conduct gradient ascent (GA) to reverse the training on undesired data. However, such a reversal is prone to catastrophic collapse, which leads to serious performance degradation in general tasks. As a solution, we propose \textit{model extrapolation} as an alternative to GA, which reaches the counterpart direction in the hypothesis space from one model given another reference model. Therefore, we leverage the original model as the reference, further train it to memorize undesired data while keeping prediction consistency on the rest retained data, to obtain a \textit{memorization model}. Counterfactual as it might sound, a \textit{forget model} can be obtained via extrapolation from the memorization model to the reference model. Hence, we avoid directly acquiring the forget model using GA, but proceed with gradient descent for the memorization model, which successfully stabilizes the machine unlearning process. Our model extrapolation is simple and efficient to implement, and it can also effectively converge throughout training to achieve improved unlearning performance.
\end{abstract}

\section{Introduction}
\label{sec:introduction}
The ground-breaking achievement of Large Language Models (LLMs) has significantly boosted both industrial development and human lives~\cite{touvron2023llama,brown2020language,hu2023llm,huang2024machine,wang2024noisegpt}. Through the integration of a tremendous amount of data, LLMs perform as knowledge experts to provide general assistance that is specialized to personal requirements. However, such a close machine-human interaction that is based on large-scale data utilization raises serious concerns about privacy~\cite{yao2024survey}, safety~\cite{wei2023jailbroken,linunderstanding,lin2025force}, and ethical issues~\cite{li2024wmdp, motoki2024more, karamolegkou2023copyright,huang2025trustworthy}. Therefore, Machine Unlearning (MU)~\cite{bourtoule2021machine, cao2015towards, fan2023salun, jia2023model} has raised abundant attention, aiming to protect private information, remove malicious data, and trustworthiness~\cite{liang2022advances}.

A naive approach to model unlearning retrains a foundation model from scratch without the undesired data, but this is prohibitively expensive, often costing millions of dollars for LLMs such as Llama~\cite{touvron2023llama}, Phi~\cite{abdin2024phi}, and OPT~\cite{zhang2022opt}. A more practical alternative reverses the training process via gradient ascent (GA)~\cite{li2024wmdp, hsieh2019classification, jia2023model, chen2023unlearn, wang2023kga, yao2024large}. In practice, the training data are split into a \textit{retain set} to preserve desired knowledge and a \textit{forget set} containing information to be removed. MU is then achieved by maximizing the loss on the forget set—\textit{i.e.}, applying GA on undesired data in contrast to GD on the retain set.

However, gradient ascent could be problematic in practice. Maini et al.~\cite{maini2024tofu} demonstrate that even though GA can improve forget quality, it sacrifices the model utility on real-world tasks. Moreover, Zhang et al.~\cite{zhang2024negative} showed that GA leads to catastrophic collapse, which seriously sabotages the training stability and deviates the models from their initial reference models. As a result, the generalization performance under real-world tasks~\cite{huang2023robust,huang2023harnessing,huang2021universal,li2024dynamic,hong2024improving} degrades drastically after a certain stage of training without gaining satisfactory improvement in forget quality. In order to solve the deficiencies of GA, most of the existing works~\cite{rafailov2023direct, maini2024tofu, zhang2024negative} conduct reweighting for the losses for forgetting and retaining, aiming to trade off the forget quality and model utility. However, no matter how to reweight the losses, the nature of GA will always pose a negative effect on training stability and model utility.

\begin{figure*}
    \centering
    \includegraphics[width=\linewidth]{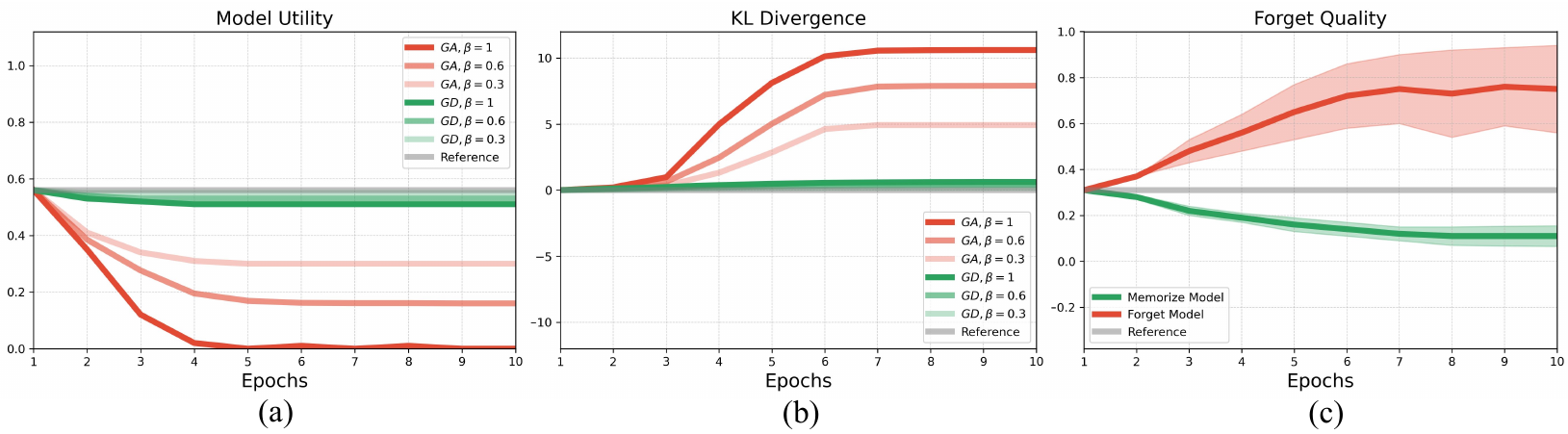}
    \caption{(a) Effect of gradient ascent and gradient descent on model utility under various reweighting levels. (b) Effect of gradient ascent and gradient descent on divergence between training and reference models under various reweighting levels. (c) Comparison of forget quality between the forget model and the memorize model.}
    \label{fig:motivation}
\end{figure*}

To analyze the effects of GA, standard GD, and loss reweighting, we perform an ablation study on the TOFU benchmark~\cite{maini2024tofu}. Using the retain and forget sets, the objectives decompose as $\mathcal{L}_{GA}=\mathcal{L}_{R}-\beta\mathcal{L}_{F}$ for GA, which minimizes risk on the retain set while maximizing risk on the forget set, and $\mathcal{L}_{GD}=\mathcal{L}_{R}+\beta\mathcal{L}_{F}$ for GD, where $\beta$ controls the forget-set contribution. Figure~\ref{fig:motivation} reports model utility on the TOFU real-world set~\cite{maini2024tofu} (a) and deviation from the reference model (b). Compared to GD, GA consistently causes larger performance degradation and greater deviation from the reference model across all $\beta$, whereas GD remains stable during training. This motivates the question of whether model unlearning can be achieved using GD alone, without relying on GA.

We propose a counter-strategy that avoids GA by enhancing memorization on the forget set using GD, yielding a memorization model that serves as the dual of the desired forget model. Using the initial model as a reference, we perform \emph{MOdel eXtrapolation} (MOX) to traverse the counterpart direction from the memorization model toward the reference, producing a forget model within the same hypothesis space that effectively removes undesired knowledge. As shown in Fig.~\ref{fig:motivation} (c), encouraging memorization degrades forget quality, whereas the extrapolated counterpart significantly improves forgetting. This approach stabilizes MU training using only GD, while remaining computationally efficient and dynamically applicable throughout training. Extensive experiments on TOFU~\cite{maini2024tofu} and MUSE~\cite{shi2024muse}, together with analytical studies, validate the effectiveness of MOX against state-of-the-art MU methods and provide further insights.

In Summary, our contribution is threefold:
\begin{itemize}
    \item We stabilize the MU training by avoiding GA, which provides a novel direction for real-world practices to harmonize forget quality and model utility.
    \item We propose a novel methodology named model extrapolation, which can be effectively and efficiently deployed into the training process to achieve MU.
    \item Experimental improvements and practical insights are contributed to the advancement of the MU community.
\end{itemize}

\section{Methodology}
\label{sec:methodology}
In this section, we elucidate the proposed MOdel eXtrapolation (MOX) for MU. First, we formulate the problem setting and with existing baselines. Then, we introduce our methodology, MOX. Finally, we describe the practice of MOX in training and discuss its advantages for unlearning.

\subsection{Preliminaries}
In MU, we are given a retain set $\mathcal{D}_R=\{(x_i, y_i)\}_{i=1}^m$ with $m$ instances with each including data $x_i$ and label $y_i$, and a forget set $\mathcal{D}_F=\{(x_i, y_i)\}_{i=m+1}^{m+n}$ which contains $n$ instances. Both $\mathcal{D}_F$ and $\mathcal{D}_R$ belong to our training set $\mathcal{D}$. For MU tasks with no target during the unlearning process, termed non-targeted MU, they only focus on forgetting the knowledge. On the contrary, targeted MU that aims to forget the knowledge, and meanwhile enforces outputting a certain target $\tilde{y_i}$. We denote the unlearning model $h_{\theta}(\cdot)$, which is parameterized by $\theta$. Under the framework of LLMs, the input data is a sequence of tokens $x=(t_1, t_2, \ldots, t_l)$, which is forwarded into the unlearning model for next-token prediction on $y$: $h_{\theta}(y|x)=\prod_{i=1}^{l} h_{\theta}(t_i|t_{<i})$. The learning objective for training LLMs commonly employs Cross-Entropy Loss $\mathcal{L}_{CE}(x, y, \theta)=\log (h_{\theta}(y|x))$, through which we can obtain the initial model, \textit{i.e.}, reference model $\theta_{ref}$. The performance of MU is commonly evaluated via the prediction quality on both the forget quality and model utility~\cite{maini2024tofu} from the forget and the retain sets, respectively.

\paragraph{Gradient Ascent (GA)~\cite{maini2024tofu}} is a straightforward solution to reverse the training process that minimizes the Cross-Entropy Loss, which is formulated as:
\begin{equation}
    \mathcal{L}_{GA}=-\frac{1}{n}\sum_{(x_i, y_i)\in\mathcal{D}_F}\mathcal{L}_{CE}(x_i, y_i, \theta).
    \label{eq:gradient_ascent}
\end{equation}
Intuitively, gradient ascent aims to deviate the learning process from the original direction. Note that the loss is computed only on the forget set; thus, the performance on the retain set would be affected and lead to degradation. Moreover, such a maximization is unbounded, which could lead to serious collapse~\cite{maini2024tofu, zhang2024negative}. 

\paragraph{Gradient Ascent with Difference (GAD)~\cite{liu2022continual}} improves gradient ascent via introducing a regularization term on the retain set:
\begin{align}
    \mathcal{L}_{GAD}&=\frac{1}{m}\sum_{(x_i, y_i)\in\mathcal{D}_R}\mathcal{L}_{CE}(x_i, y_i, \theta)\notag\\
    &-\frac{1}{n}\sum_{(x_j, y_j)\in\mathcal{D}_F}\mathcal{L}_{CE}(x_j, y_j, \theta).
    \label{eq:gradient_ascent_difference}
\end{align}
The difference is denoted by the opposite gradient direction between the retain and forget sets. By adding the regularization, it acts as a compromise between forgetting and model utility on the retain set. However, it is hard to balance the two losses in practice. Moreover, the two losses could conflict with each other during training~\cite{maini2024tofu, wang2025rethinking}.

\paragraph{Negative Preference Optimization (NPO) ~\cite{zhang2024negative}} takes inspiration from Direct Preference Optimization~\cite{rafailov2023direct}, which basically conducts weighted GA instead of using Cross-Entropy optimization:
\begin{equation}
    \mathcal{L}_{NPO}=-\frac{2}{n\beta}\sum_{(x_i, y_i)\in\mathcal{D}_F}\log\sigma(-\beta\log(\frac{h_{\theta}(y_i|x_i)}{h_{\theta_{ref}}(y_i|x_i)})),
    \label{eq:npo}
\end{equation}
where $\beta$ is a temperature hyperparameter, $\sigma(\cdot)$ is the sigmoid function, and $h_{\theta_{ref}}(\cdot)$ denotes the forward function of the reference model. We can see that the temperature acts as a scaling parameter to reweight the importance of each instance. However, it still suffers from inferior model utility and instability due to the effect of GA-based optimization~\cite{wang2024llm, wang2025rethinking}. However, despite the deficiencies of GA, the advantages of the existing methods can still be leveraged. Therefore, we propose a counter-strategy that achieves unlearning via memorization.

\begin{figure*}
    \centering
    \includegraphics[width=\linewidth]{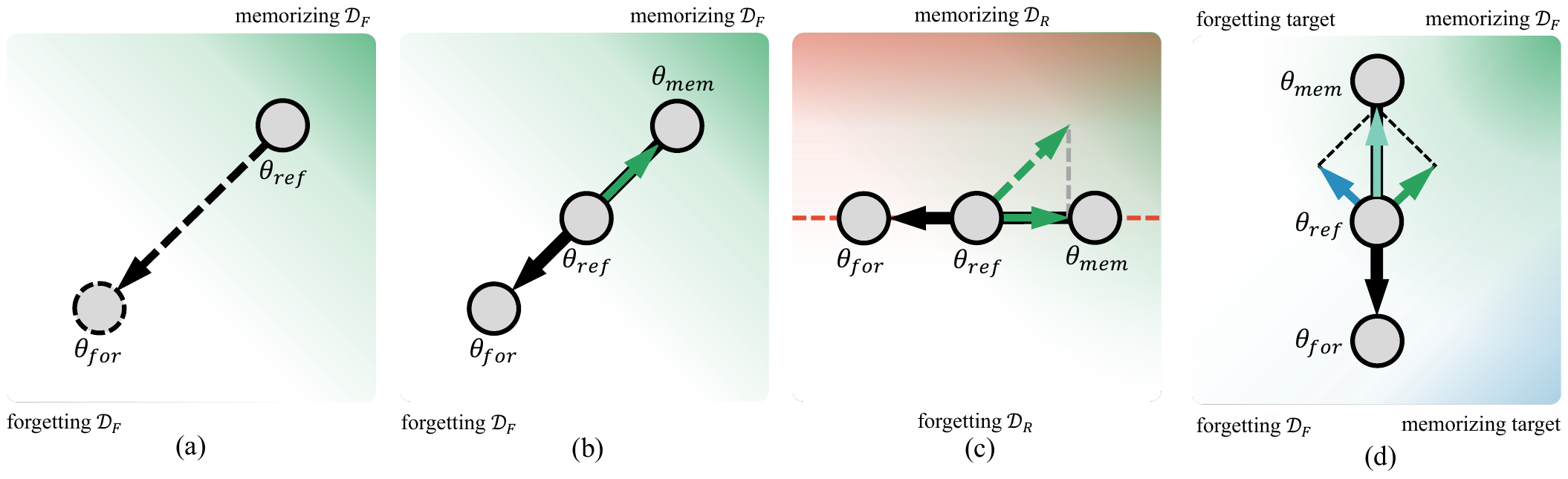}
    \caption{Illustration of MOX. Color intensity indicates dataset fit, colored arrows denote learning directions, and black arrows indicate model extrapolation. Directly deriving the forget model $\theta_{for}$ from the reference model $\theta_{ref}$ via gradient ascent is infeasible, as it reverses pre-training and leads to optimization failures. Instead, we apply gradient descent to memorize $\mathcal{D}_F$, obtaining a memorization model $\theta_{mem}$, and then extrapolate to produce $\theta_{for}$ that effectively forgets $\mathcal{D}_F$. To preserve utility, a KL-divergence constraint enforces consistency between $\theta_{ref}$ and $\theta_{mem}$, improving the utility of $\theta_{for}$. For targeted unlearning, an additional forgetting loss—compatible with pre-training—is combined with the memorization loss to perform MOX and obtain $\theta_{for}$.}
    \label{fig:illustration}
\end{figure*}

\subsection{Memorize to Forget}
Given the previous findings that using GA for forgetting is problematic, and using GD can only achieve memorization instead of our unlearning goal. Therefore, understanding the relationship between GA and GD is vital to accomplish further ``memorize to forget''. Ilharco et al.~\citep{ilharco2022editing} proposed a \textit{task vector} that demonstrates negating a task vector results in reduced performance on the task, i.e., forgetting the task. Specifically, given a fine-tuned model $f_{new}$, applying a negative vector $-f_{new}$ to the pre-trained model $f$ can extrapolate between $f$ and $f_{new}$. Hence, it is possible to first enhance memorization, then use extrapolation to achieve forgetting, as explained by our MOX below.

\subsection{Model Extrapolation (MOX)}
Based on the above discussion, we suggest a \textit{irreversible gradient} criterion:
\begin{definition}[irreversible gradient]
Given a pre-trained model $\theta$ that has been trained through an optimization task $\Psi$, any downstream fine-tuning task with a gradient that reverses the gradient of $\Psi$ should be avoided.
\label{def}
\end{definition}
The above definition denotes that reversing the gradient via GA is not preferred, which is an important criterion for designing our methodology. The intuition of our methodology is shown in Fig.~\ref{fig:illustration}, in which (a) denotes that directly obtaining the forget model $\theta_{for}$ via GA is not feasible due to the irreversible gradient criteria, as $\mathcal{D}_F\subset\mathcal{D}$, and forgetting $\mathcal{D}_F$ goes against the pre-trianing task aiming to fit $\mathcal{D}_F$.

Therefore, we propose to intensify memorization on the forget set, and meanwhile maintain the prediction consistency on the retain set between the training model and the reference model. Hence, we can obtain a memorization model $\theta_{mem}$ by leveraging Cross-Entropy Loss:
\begin{align}
    \mathcal{L}_{mem}&=\frac{1}{n}\sum_{(x_j, y_j)\in\mathcal{D}_F}\mathcal{L}_{CE}(x_j, y_j, \theta)\notag\\
    &+\frac{1}{m}\sum_{(x_i, y_i)\in\mathcal{D}_R}\text{KL}(h_{\theta_{ref}}(y|x)\|h_{\theta}(y|x)).
    \label{eq:memorization_GA}
\end{align}
Similarly, we can follow NPO to leverage the preference optimization objective:
\begin{align}
    \mathcal{L}_{mem}&=\frac{2}{n\beta}\sum_{(x_i, y_i)\in\mathcal{D}_F}\log\sigma(-\beta\log(\frac{h_{\theta}(y|x)}{h_{\theta_{ref}}(y|x)}))\notag\\
    &+\frac{1}{m}\sum_{(x_i, y_i)\in\mathcal{D}_R}\text{KL}(h_{\theta_{ref}}(y|x)\|h_{\theta}(y|x)).
    \label{eq:memorization_NPO}
\end{align}
In both objectives, the first terms aim to enhance memorization of the forget set. As shown in Fig.~\ref{fig:illustration} (b), although we can already obtain a memorization model and effectively obtain a forget model that achieves the forgetting goal (see Section \ref{sec:experiments}), we hope to further maintain the model utility. Hence, we introduce a second term in both Eqs.~\eqref{eq:memorization_GA} and~\eqref{eq:memorization_NPO} which is the KL-divergence regularization that aims to minimize the Kullback-Leibler (KL) divergence between the predictions on $\mathcal{D}_R$ of the reference model $\theta_{ref}$ and the training model $\theta$. As shown in Fig.~\ref{fig:illustration} (c), the learning direction of memorization is constrained to maintain the performance on $\mathcal{D}_R$\footnote{Several works~\cite{wang2025rethinking, maini2024tofu} aim to achieve enhanced model utility performance on $\mathcal{D}^R$ and forgetting simultaneously. However, we only focus on improving forgetting quality and maintaining model utility, because the utility performance can be effectively enhanced via fine-tuning the pre-trained model~\cite{hu2023llm,zhao2024galore}. Further, our method is built on the enhanced LLMs for MU without sacrificing the utility performance. Thus, only focusing on forgetting and maintaining utility is sufficient for MU.}, which helps maintain the general utility performance for the memorization model. By optimizing only with GD, we show that it avoids collapse.

\begin{definition}[Collapse]
We say a parameter $\theta$ is collapsed on a subset $D_{sub}$ if the model outputs a degenerate distribution that is nearly constant across all contexts. Formally,
\begin{equation}
    \exists y^* \quad \text{s.t.} \quad 
    p_\theta(y^* \mid c) \approx 1 
    \ \text{for almost all contexts } c.
\end{equation}
More generally, collapse can be quantified by the average KL divergence between the empirical conditional distributions $q(\cdot \mid c_i)$ and the model predictions:
\begin{equation}
    \mathrm{Collapse}(\theta):=\frac{1}{n}\sum_{i=1}^n 
    D_{\mathrm{KL}}\left(
        q(\cdot \mid c_i) \middle\| p_\theta(\cdot \mid c_i)
    \right).
\end{equation}
Large values of $\mathrm{Collapse}(\theta)$ correspond to assigning near-delta mass to a small set of tokens regardless of the input.
\end{definition}
Under the following assumptions:
\begin{assumption}
The loss $L$ given a model $\theta$, it satisfies:
\begin{itemize}
    \item Lower bounded loss: $L(\theta)\ge 0$ for all $\theta$.
    \item $L_s$-smooth: $\|\nabla L(\theta)-\nabla L(\theta')\|\le L_s \|\theta-\theta'\|$.
    \item Data diversity: For each context $c_i$, the empirical conditional distribution $q(\cdot\mid c_i)$ has non-zero entropy.
    \item Polyak--Łojasiewicz (PL) condition: We say that $L(\theta)$ satisfies the PL condition with parameter $\mu > 0$ if $\frac{1}{2}\|\nabla L(\theta)\|^2\ge\mu \big( L(\theta) - L^\star \big)$, where $L^\star=\inf_{\theta} L(\theta).$
\end{itemize}
\end{assumption}
\begin{theorem}
\label{theorem}
    By avoiding the irreversible gradient via GD on $D_{sub}$, the model converges without collapse.
\end{theorem}
Please see the proof in the appendix.

As a result, the obtained model $\theta_{mem}$ retains the preferred knowledge but even fits intensely to the forget set $\mathcal{D}_F$. However, our desired model $\theta_{for}$ should maintain the knowledge from $\mathcal{D}_R$ and forget the undesired knowledge in $\mathcal{D}_F$. Hence, we can reasonably claim that $\theta_{mem}$ is a counterpart of $\theta_{for}$, where the reference model $\theta_{ref}$ plays a critical role between memorizing and forgetting, also shown empirically in Sec.\ref{sec:experiments}. To further obtain the forget model $\theta_{for}$, we conduct MOX to extrapolate from $\theta_{mem}$ to $\theta_{ref}$ via the following:
\begin{equation}
    \theta_{for}:=(1+\alpha)\theta_{ref}-\alpha\theta_{mem}, \ \alpha\in\mathbb{R}^+,
    \label{eq:non_convex}
\end{equation}
where the hyper-parameter $\alpha$ controls the extrapolation strengths for MOX. If $\alpha$ is set to a high value, the forget quality can be enhanced. We show in experiments (Section \ref{sec:experiments}) that the performance improvement of forgetting is significant by reasonably increasing $\alpha$. Although we found that extreme values of $\alpha$ would lead to a decrease in model utility, the decrease is much slighter than catastrophic collapse in GA, thus, it is easy to detect and avoid. 

Intuitively, by deriving Eq.~\ref{eq:non_convex}, we have $\theta_{for}:=\theta_{ref}+\alpha(\theta_{ref}-\theta_{mem})$, where the task vector $\theta_{ref}-\theta_{mem}$ is scaled by $\alpha$ and added to $\theta_{ref}$. Since the processing of obtaining $\theta_{mem}$ from $\theta_{ref}$ is memorizing, the opposite direction $\theta_{ref}-\theta_{mem}$ is forgetting, as demonstrated in Ilharco et al.~\citep{ilharco2022editing}. Thus, the difference between our forget model and the reference model $\theta_{for}-\theta_{ref}$ can be interpreted by forgetting the knowledge that $\theta_{mem}$ has memorized. Such a reversal is feasible based on the well-trained LLMs that are approximately linear locally. Therefore, it can be explained through Neural Tangent Kernel~\cite{jacot2018neural}, where 
$f(\theta) \approx f(\theta_0) + \nabla f(\theta_0)^T (\theta - \theta_0)$. Thus, steering models via a linear transformation is feasible in practice. Next, we present two practical implementations of MOX with advanced effectiveness.

\paragraph{Targeted Unlearning} is one MU learning problem where a target is provided for examples to be forgotten. As shown in Fig. \ref{fig:illustration} (d), we hope to forget $\mathcal{D}_F$ and memorize the target simultaneously. Although the previous discussion is under non-targeted scenarios, we can still tackle targeted MU by effectively introducing a targeted unlearning term:
\begin{equation}
    \mathcal{L}=\mathcal{L}_{mem}-\frac{1}{m}\sum_{(x_i, y_i)\in\mathcal{D}_R}\mathcal{L}_{CE}(x_i, \tilde{y}_i, \theta),
    \label{eq:targeted_mu}
\end{equation}
where $\tilde{y}_i$ is the target. Similar to Eq.~\eqref{eq:memorization_NPO}, we can apply preference optimization for realization. Through training to obtain $\theta_{mem}$, MOX can be applied to targeted MU cases. Note that directly maximizing the risk does not conflict with pre-training, thus it can be effectively applied.

\paragraph{Momentum Extrapolation} is an effective technique for improving generalization and forget quality, thanks to the computation efficiency of our MOX. Since the extrapolation process can be conducted on-the-fly, we can gradually update the forget model by ensembling historical versions throughout the training:
\begin{equation}
    \theta_{for}^t:=\eta\theta_{for}^t+(1-\eta)\theta_{for}^{t-1},
    \label{eq:momentum_mox}
\end{equation}
where $\eta$ is a momentum coefficient~\cite{he2020momentum, tarvainen2017mean} and it is normally set to $0.675$, and $\theta_{for}^{t-1}$ is the history ensemble. Once obtaining the current $\theta_{for}^t$, we update it via Eq.~\eqref{eq:momentum_mox} to incorporate historical weight, which largely benefits generalization performance for both forget quality and model utility. Next, we conduct experiments to validate our MOX.
\paragraph{Discussion}
Our MOX holds several non-negligible advantages compared to some typical MU strategies: 1) Computation Stability: We identify the detrimental failure of GA, and we only use GD for training. Thus, there is no concern for catastrophic collapse or task conflicts. 2) Adaptability: Our learning target can be implemented using various objectives, benefiting from the recent advancement in MU and LLM training. 3) Efficiency: Our MOX is simple to implement, and the forget model can be computed directly from subtracting model parameters, which enables its dynamic deployment during training. 

Compared to task vectors~\cite{ilharco2022editing}, MOX is more flexible on the extrapolation strength, thus significantly enhancing the forget quality. Additionally, task vectors are limited to only forgetting, which ignores the knowledge retaining goal of machine unlearning. Our approach considers a retain regularization, which largely enhances the model utility and can be adapted to both targeted and untargeted settings. More importantly, the failure of gradient ascent is diagnosed for training stability concerns, and existing unlearning methods neglect or underestimate such a problem. Using MOX as an alternative solution. Furthermore, we employ extrapolation in a very efficient offline manner, so that it can be conducted in any phase of training with additional momentum update.

\begin{table*}[t]
\centering
\small
\setlength{\tabcolsep}{10pt}
\caption{Comparison of MOX and other baseline methods on TOFU benchmark. The top two best performances in each column are highlighted in \textbf{bold}.}
\begin{tabular}{l|cccc|cccc}
\toprule
\multicolumn{1}{c|}{\textbf{Base LLM}} & \multicolumn{4}{c|}{\textbf{Llama2-7B}} & \multicolumn{4}{c}{\textbf{Phi-1.5B}} \\
\midrule
\multicolumn{1}{c|}{\textbf{Metric}} & FQ($\uparrow$) & MU($\uparrow$) & F-RL($\downarrow$) & R-RL($\uparrow$) & FQ($\uparrow$) & MU($\uparrow$) & F-RL($\downarrow$) & R-RL($\uparrow$) \\
\midrule
Original LLM & 0.0000 & 0.6346 & 0.9851 & 0.9833 & 0.0013 & 0.5184 & 0.9607 & 0.9199 \\
Retrained LLM & 1.0000 & 0.6267 & 0.4080 & 0.9833 & 1.000 & 0.5233 & 0.4272 & 0.9269 \\
\midrule
GA  & 0.0143 & 0.6333 & 0.4862 & 0.9008 & 0.0213 & 0.5069 & 0.5114 & 0.8048 \\
KL  & 0.0168 & 0.6300 & 0.5281 & 0.9398 & 0.0120 & 0.5047 & 0.5059 & 0.8109 \\
GAD  & 0.0268 & 0.6320 & 0.4773 & 0.8912 & 0.0215 & 0.5110 & 0.4996 & 0.8496 \\
PO  & 0.0541 & 0.6308 & \textbf{0.3640} & 0.8811 & 0.0286 & 0.5127 & 0.3170 & 0.7468 \\
\midrule
LLMU & 0.0541 & 0.6337 & 0.4480 & 0.8865 & 0.0286 & 0.5110 & \textbf{0.3058} & 0.7270 \\
DPO  & 0.0541 & 0.6359 & 0.5860 & 0.8852 & 0.0521 & 0.5125 & 0.3437 & 0.7349 \\
NPO  & 0.0068 & 0.6321 & 0.4632 & 0.8950 & 0.0030 & 0.5057 & 0.5196 & 0.8000 \\
AltPO & 0.0120 & 0.6432 & 0.4650 & 0.8953 & 0.0231 & 0.5312 & 0.5200 & 0.8288 \\
SimNPO & 0.0172 & 0.6450 & 0.4601 & 0.9201 & 0.0315 & 0.5665 & 0.5135 & 0.8514 \\
RMU & 0.0211 & 0.6378 & 0.4645 & 0.9032 & 0.0387 & 0.5528 & 0.5150 & 0.8488 \\
TV   & 0.0069 & 0.6340 & 0.4512 & \textbf{0.9810} & 0.0156 & 0.5012 & 0.4366 & 0.8810 \\
\midrule
MOX ($\alpha=0.5$) & 0.0146 & 0.6305 & 0.4812 & \textbf{0.9810} & 0.0163 & 0.5002 & 0.4512 & \textbf{0.9200} \\
MOX ($\alpha=1.0$) & 0.0182 & 0.6358 & 0.4732 & 0.9788 & 0.0180 & 0.5026 & 0.4366 & \textbf{0.9120} \\
MOX ($\alpha=2.0$) & 0.0256 & 0.6410 & 0.4555 & 0.9701 & 0.0364 & 0.5012 & 0.4330 & 0.8928 \\
MOX ($\alpha=4.0$) & 0.0625 & \textbf{0.6504} & 0.4697 & 0.9653 & \textbf{0.0582} & \textbf{0.5219} & 0.3138 & 0.8810 \\
MOX ($\alpha=8.0$) & 0.0319 & 0.6420 & 0.4658 & 0.9016 & 0.0340 & 0.5150 & 0.3436 & 0.8562 \\
MOX (\scriptsize targeted) & \textbf{0.0677} & 0.6412 & 0.4788 & 0.9710 & 0.0328 & 0.5012 & 0.3366 & 0.8858 \\
MOX (\scriptsize momentum) & \textbf{0.0680} & \textbf{0.6528} & \textbf{0.4410} & \textbf{0.9802} & \textbf{0.0598} & \textbf{0.5510} & \textbf{0.3120} & 0.8988 \\
\bottomrule
\end{tabular}
\label{tab:tofu}
\end{table*}

\section{Experiments}
\label{sec:experiments}
In the experiments, we first compare our method with several typical baseline methods to validate its effectiveness. Then, we conduct an ablation study to decompose and investigate each module of our MOX. Moreover, we analyze the hyper-parameters to understand our experimental design. Further, we study the stability of MOX over various datasets with different forget sizes during training.

\paragraph{Baseline Methods.}
We compare MOX to various strong LLM unlearning techniques, namely, Gradient Ascent (GA)~\cite{maini2024tofu}, KL-divergence minimization (KL)~\cite{maini2024tofu}, Gradient Ascent with Difference (GAD)~\cite{liu2022continual}, NPO~\cite{zhang2024negative}, AltPO~\cite{mekala2024alternate}, SimNPO~\cite{fan2024simplicity}, and RMU~\cite{li2024wmdp}, Task Vectors (TV)~\cite{ilharco2022editing}, LLM Unlearning (LLMU)~\cite{yao2024large}, and Who's Harry Potter (WHP)~\cite{eldan2023who}. Additionally, we add Preference Optimization (PO)~\cite{maini2024tofu} and Direct Preference Optimization (DPO)~\cite{rafailov2023direct} using target data to conduct targeted unlearning. Specifically, we use the template ``I don't know the answer'' as the target data, as done in Rafailov et al.~\citep{rafailov2023direct}. For our method, we tune the values of extrapolation strength $\alpha$ as $0.5$, $1.0$, $2.0$, $4.0$, and $8.0$. Further, we also conduct targeted unlearning for MOX. Additionally, we investigate the performance of MOX with momentum extrapolation to further justify its effectiveness. Our evaluation follows Wang et al.~\citep{wang2024llm}, where the evaluation is reported based on the last epoch, instead of the best performance during training~\cite{zhang2024negative}.

\begin{table}[t]
\centering
\small
\setlength{\tabcolsep}{17pt}
\caption{Comparison of MOX and baseline methods on MUSE benchmark. The top two best performances in each column are highlighted in \textbf{bold}.}
\begin{tabular}{l|cccc}
\toprule
\textbf{Base LLM} & \multicolumn{4}{c}{\textbf{Llama2-7B}} \\
\midrule
\multirow{2}{*}{\textbf{Metric}} & No Verb. Mem. & No Know. Mem. & Util. Preserv. & No Priv. Leak.   \\
\multicolumn{1}{c|}{} & 
 Verb. on $\mathcal{D}_F$ \!($\downarrow$)\! & 
 Know. on $\mathcal{D}_F$ \!($\downarrow$)\! & 
 Know. on $\mathcal{D}_R$ \!($\uparrow$) & 
 Priv.\! ($\!\in\! [\!-5\%\!,\! 5\%\!]$) 
\\
\midrule
Origin & 58.4 & 63.9 & 55.2 & -99.8 \\
Retain & 20.8 & 33.1 & 55.0 & 0.0 \\
\midrule
GA  & \textbf{0.0} & \textbf{0.0} & 0.0 & \textbf{17.0} \\
KL  & 27.4 & 50.2 & 44.8 & -96.1 \\
GAD  & 4.9 & 31.0 & 27.3 & 108.1 \\
PO  & 2.3 & 21.8 & 16.1 & 109.6 \\
\midrule
NPO  & \textbf{0.0} & \textbf{0.0} & 0.0 & 24.4 \\
TV  & 57.2 & 66.2 & \textbf{55.8} & -99.8 \\
WHP & 19.7 & 21.2 & 28.3 & 109.6 \\
\midrule
MOX \tiny ($\alpha\!=\!0.5$) & 36.5 & 38.6 & \textbf{56.2} & -93.1 \\
MOX \tiny ($\alpha\!=\!1.0$) & 27.7 & 29.4 & 55.6 & -58.8 \\
MOX \tiny ($\alpha\!=\!2.0$) & 18.2 & 19.8 & 55.2 & -32.0 \\
MOX \tiny ($\alpha\!=\!4.0$) & 1.2 & 1.6 & 54.9 & -19.8 \\
MOX \tiny ($\alpha\!=\!8.0$) & 0.8 & 1.1 & 49.5 & 35.8 \\
MOX \tiny (targeted) & \textbf{0.2} & 1.5 & 53.6 & 26.0 \\
MOX \tiny (moment.) & \textbf{0.2} & \textbf{0.8} & 54.8 & \textbf{-18.4} \\
\bottomrule
\end{tabular}
\label{tab:muse}
\end{table}

\begin{table*}[t]
\centering
\small
\setlength{\tabcolsep}{6pt} 
\caption{Ablation study with various modules of MOX on TOFU dataset.}
\begin{tabular}{l|ccc|ccc|ccc|ccc}
\toprule
\multirow{2}{*}{\textbf{Metric}} & 
\multicolumn{3}{c|}{\textbf{Real Authors}} &
\multicolumn{3}{c|}{\textbf{Real World}} &
\multicolumn{3}{c|}{\textbf{Retain Set}} &
\multicolumn{3}{c}{\textbf{Forget Set}} \\
& RL($\uparrow$) & P($\uparrow$) & TR($\uparrow$) & RL($\uparrow$) & P($\uparrow$) & TR($\uparrow$) & RL($\uparrow$) & P($\uparrow$) & TR($\uparrow$) & RL($\downarrow$) & P($\downarrow$) & TR($\uparrow$) \\
\midrule
GA                                & 0.91 & 0.47 & 0.63 & 0.89 & 0.44 & 0.54 & 0.85 & 0.94 & 0.42 & 0.42 & 0.78 & 0.65 \\
MOX (\scriptsize GD)              & 0.91 & 0.46 & 0.63 & 0.89 & 0.45 & 0.55 & 0.86 & 0.95 & 0.43 & 0.43 & 0.77 & 0.65 \\
MOX (\scriptsize GD+KL)           & 0.92 & \textbf{0.49} & \textbf{0.64} & \textbf{0.90} & 0.46 & \textbf{0.57} & 0.87 & \textbf{0.96} & 0.45 & \textbf{0.40} & 0.75 & \textbf{0.67} \\
MOX (\scriptsize GD+target)       & 0.92 & 0.47 & 0.63 & 0.89 & 0.45 & 0.55 & 0.86 & 0.94 & 0.43 & \textbf{0.40} & 0.75 & 0.66 \\
MOX (\scriptsize GD+KL+target)    & \textbf{0.93} & \textbf{0.49} & \textbf{0.64} & \textbf{0.90} & \textbf{0.47} & \textbf{0.57} & \textbf{0.88} & \textbf{0.96} & \textbf{0.46} & 0.41 & \textbf{0.74} & 0.66 \\
\midrule
NPO                               & 0.92 & 0.48 & 0.64 & 0.88 & 0.45 & 0.55 & 0.85 & 0.96 & 0.45 & 0.41 & 0.75 & 0.66 \\
MOX (\scriptsize PO)              & 0.93 & 0.48 & 0.64 & 0.89 & 0.45 & 0.56 & 0.86 & \textbf{0.97} & 0.46 & 0.40 & 0.74 & 0.67 \\
MOX (\scriptsize PO+KL)           & \textbf{0.94} & 0.48 & \textbf{0.66} & 0.89 & 0.46 & \textbf{0.57} & 0.87 & \textbf{0.97} & \textbf{0.48} & 0.39 & \textbf{0.72} & \textbf{0.68} \\
MOX (\scriptsize PO+target)       & 0.93 & \textbf{0.49} & 0.65 & 0.89 & 0.46 & \textbf{0.57} & 0.86 & 0.96 & 0.47 & \textbf{0.39} & 0.73 & 0.67 \\
MOX (\scriptsize PO+KL+target)    & \textbf{0.94} & \textbf{0.49} & 0.65 & \textbf{0.90} & \textbf{0.47} & \textbf{0.57} & \textbf{0.88} & \textbf{0.97} & \textbf{0.48} & \textbf{0.39} & 0.74 & \textbf{0.68} \\
\bottomrule
\end{tabular}
\label{tab:ablation}
\end{table*}

\begin{figure*}[t]
    \centering
    \includegraphics[width=\linewidth]{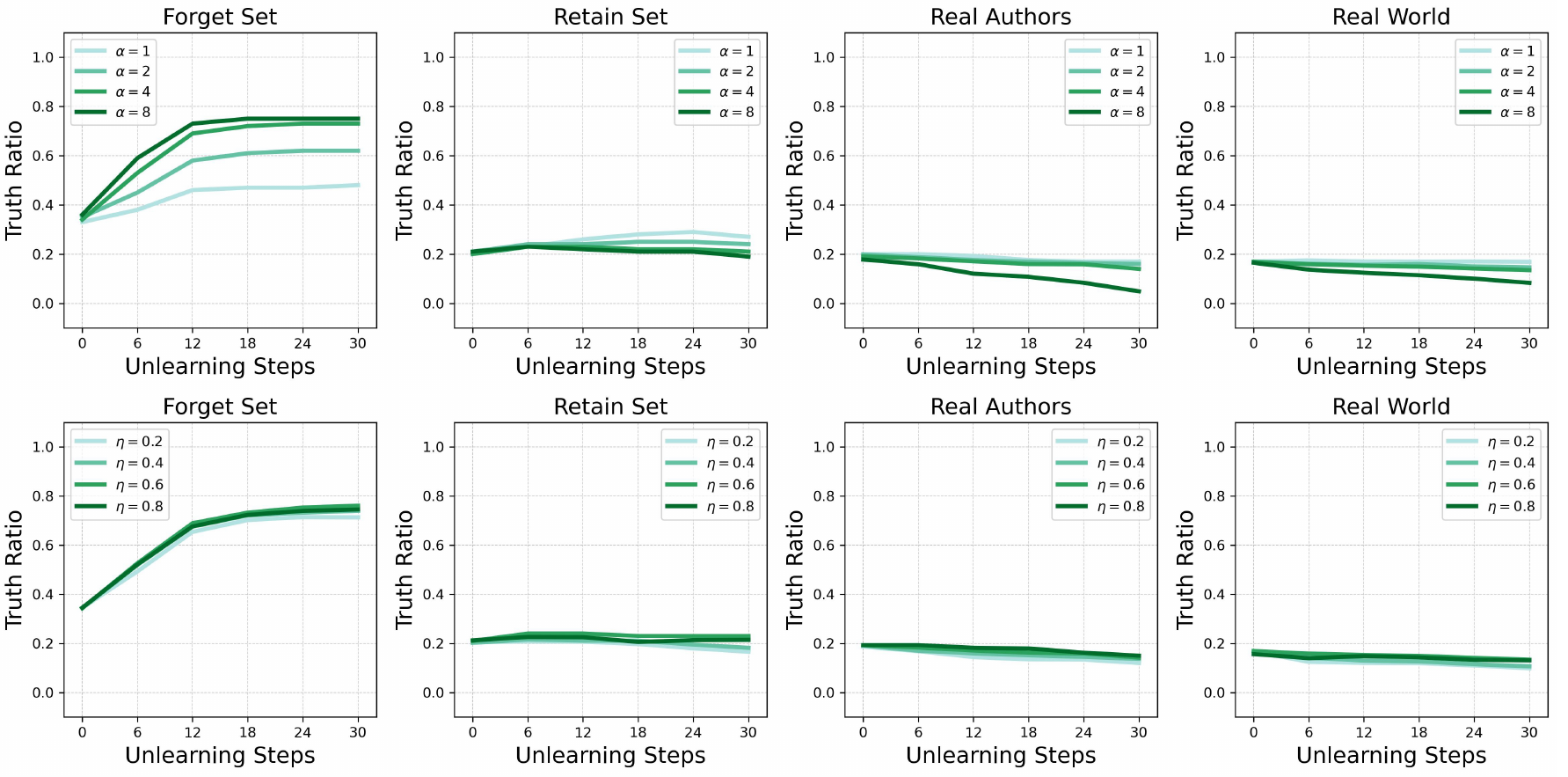}
    \caption{Parameter sensitivity analyses on $\alpha$: top row and $\eta$: bottom row.}
    \label{fig:alpha}
\end{figure*}

\paragraph{Evaluation Details.} To evaluate the MU methods, we leverage TOFU~\cite{maini2024tofu} and MUSE~\cite{shi2024muse} benchmarks. For TOFU, it contains 200 fictitious author profiles with each containing 20 question-answer pairs. To simulate the knowledge scale, there are four types of datasets in TOFU, namely Forget Set, Retain Set, Real Authors, and Real World. For evaluation on different sizes of Forget Set, TOFU provided forget ratios of $1\%$, $5\%$, and $10\%$. The evaluation metrics from TOFU mainly include Forget Quality (FQ) and Model Utility (MU). Moreover, we use ROUGE-L~\cite{lin2004rouge} on both the Forget Set and Retain Set, namely F-RL and R-RL, respectively. For MUSE, we use the New corpus that contains BBC news collected after August 2023. By following the benchmark, we evaluate four metrics, namely VerbMem on Forget Set, KnowMem on Forget Set, KnowMem on Retain Set, and PrivLeak, which denotes no verbatim memorization, no knowledge memorization, utility preservation, and no privacy leakage. All metrics are indicated with ``$\uparrow$'' or ``$\downarrow$'' to show whether a larger or smaller value leads to better performance.

\paragraph{Experimental Setup} We use two LLMs, namely Llama2-7B~\cite{touvron2023llama} and Phi-1.5B~\cite{abdin2024phi}. We use AdamW optimizer with a weight decay of 0.01 and a learning rate of 1$e-5$ for training. We use a batch size of 32 and conduct 10 epochs of unlearning training for all experiments. By following Zhang et al.~\cite{zhang2024negative}, we use linear warm-up for the learning rate in the first epoch and then linearly decay the learning rate for the rest of the training. In the following experiments, if not specified, we choose $\alpha=4$ with $\eta=0.675$ for our MOX. All experiments are conducted on two H100 GPUs.

\begin{figure*}[t]
    \centering
    \includegraphics[width=\linewidth]{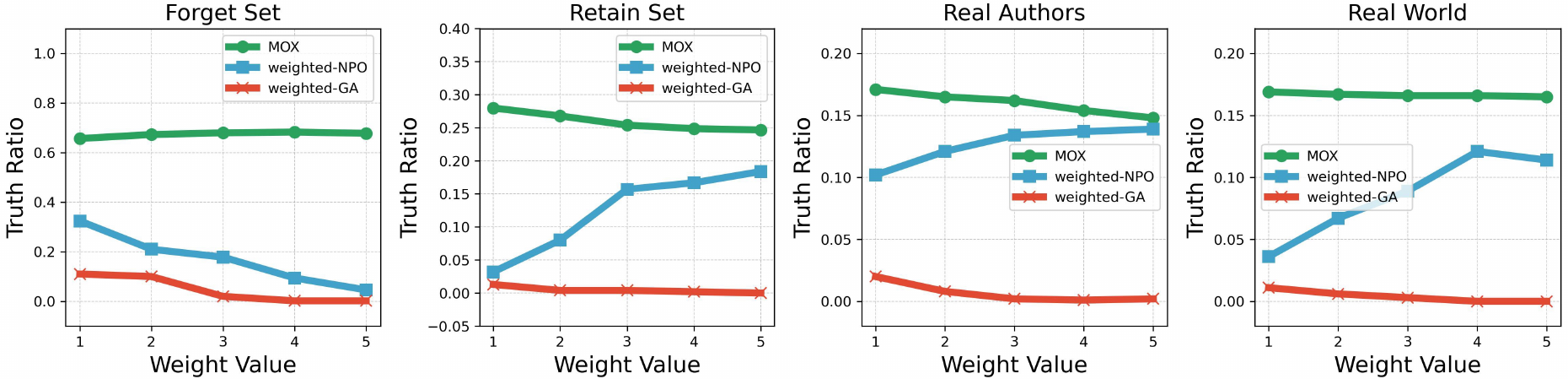}
    \caption{Performance stability under different extrapolation strengths and weight values.}
    \label{fig:stability}
\end{figure*}

\begin{figure}[t]
    \centering
    \includegraphics[width=0.6\linewidth]{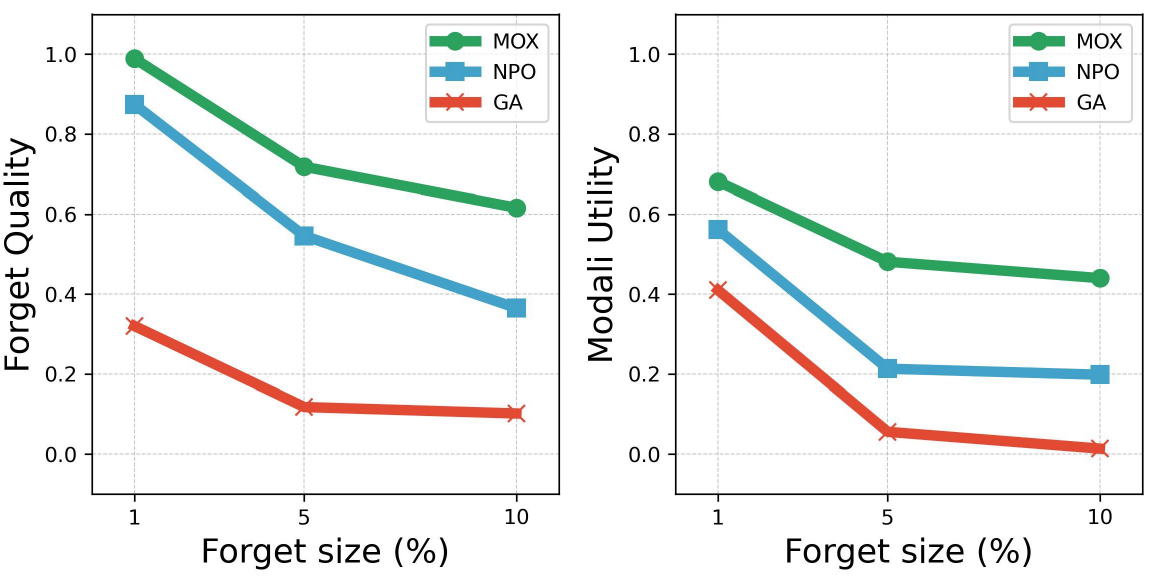}
    \caption{Performance comparison under various forget sizes.}
    \label{fig:ratio}
\end{figure}

\subsection{Performance Evaluation}
We conduct experimental comparisons between MOX with various settings and a series of strong MU baseline methods. The results on TOFU and MUSE benchmarks are shown in Tabs. \ref{tab:tofu} and \ref{tab:muse}, respectively. We can see that MOX achieves effectiveness in most scenarios compared with all baseline methods. Specifically, we observe three qualities of MOX: 1) Effective forgetting, 2) Outstanding knowledge preserving, and 3) Enhanced privacy leakage. 

In both tables, we can see that MOX with momentum extrapolation achieves the best performance, as it ensembles historical knowledge on the fly. But, as we tune the $\alpha$ value, we can achieve similar effectiveness as the momentum ensemble. For example, on TOFU with Llama2-7B, we can achieve comparable FQ performance when $\alpha$ increases to 4, and on MUSE, comparable FQ performance can be obtained with $\alpha=8$. Which denotes that proper values of $\alpha$ can control the forgetting performance to reach the best result. Moreover, we find that MOX is effective in preserving the model utility under most $\alpha$ values. For example, on MUSE, we observe the Utility Preserv. performance of MOX consistently surpasses most of the baseline methods, and it stays effective in most cases. However, when $\alpha$ achieves $8$, such an extreme value slightly degrades the knowledge preservation in the Retain Set, and yet it still outperforms most of the baselines. At last, we find that MOX can avoid privacy leakage by tuning $\alpha$ to a proper value, as we can see that when $\alpha=4$, the PrivLeak performance surpasses most of the baselines.

\setlength{\intextsep}{5pt}
\setlength{\columnsep}{10pt}
\begin{wrapfigure}{r}{0.5\linewidth}
    \centering
    \includegraphics[width=\linewidth]{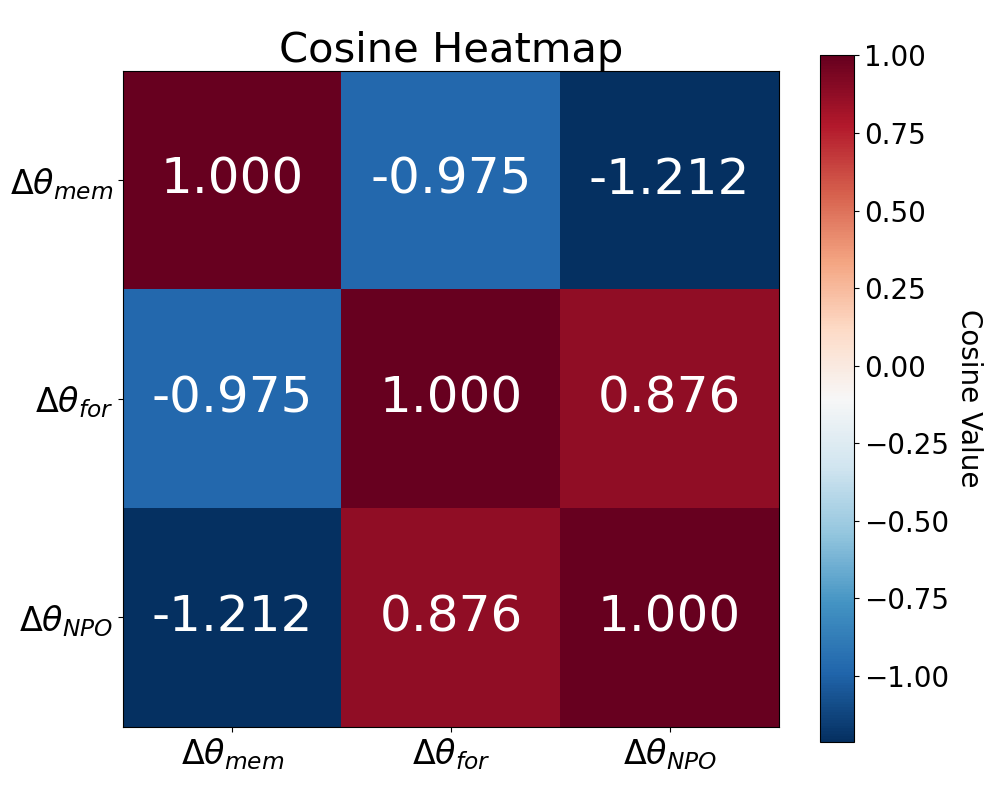}
    \caption{Cosine similarity of different learning directions.}
    \label{fig:heatmap}
\end{wrapfigure}
\subsection{Ablation and Analyzation}
\paragraph{Ablation Study.} Further, we conduct an ablation study to break down each module of MOX and understand its effectiveness. Specifically, there are four settings: 1) GD: We only conduct GD-based memorization to conduct MOX, 2) GD+KL: We conduct memorization with the KL-divergence constraint, 3) GD+target: We conduct memorization with targeted unlearning, and 4) GD+KL+target: We use all modules to realize our MOX. Moreover, we can also use PO for the memorization objective to replace GD. The result is shown in Tab. \ref{tab:ablation}. We can see that both targeted and untargeted MOX that contain ``GD+KL'' can achieve the best performance under both objective settings, which again justifies the effectiveness of our method.

\paragraph{Relationship between GA and GD}
To further justify the relationship between GA and GD, we compute their normalized model difference from $\theta_{ref}$, e.g., $\Delta \theta_{mem} = \frac{\theta_{mem} - \theta_{ref}}{\|\theta_{mem} - \theta_{ref}\|_2}$. Thus, $\Delta \theta_{mem}$ denotes the direction of $\theta_{mem}$ compared to $\theta_{ref}$, i.e., the direction of GD. Then, we use NPO that conducts GA to compute $\Delta \theta_{NPO}$. If we compute the cosine similarity between $\Delta \theta_{NPO}$ and $\Delta \theta_{mem}$, we can understand the learning direction and the relationship between GA and GD. Moreover, we compute $\Delta \theta_{for}$ using $\theta_{for}$ to understand the direction of our MOX. The results are shown in Fig.~\ref{fig:heatmap}. We can see that $\Delta \theta_{mem}$ and $\Delta \theta_{NPO}$ are almost opposite to each other. Also, $\Delta \theta_{for}$ are at the similar directions to $\Delta \theta_{NPO}$. Moreover, $\Delta \theta_{mem}$ and $\Delta \theta_{for}$ are almost opposite to each other. Thus, we can testify that GA and GD are in the opposite learning direction compared to the reference model. Also, our forget model can achieve a similar direction as GA to achieve forgetting.

\paragraph{Parameter Analyses.} Further, we vary the $\alpha$ and $\eta$ parameters over four datasets from TOFU to understand the sensitivity of MOX, as shown in Fig. \ref{fig:alpha}. We can see that larger $\alpha$ leads to better performance on the Forget Set, but it would gradually decrease the performance on the rest datasets. However, we can see that when $\alpha$ is varied from 1 to 4, the performance degradation is non-trivial; only when it is set to an extreme value like 8, the performance drop would be obvious. Hence, in general, we can set $\alpha$ to a reasonably large value for maximum gain on forgetting without losing much utility performance. As for $\eta$, we can clearly see that MOX performs consistently across various settings; hence, we can say MOX is actually insensitive to $\eta$.

\paragraph{Performance Stability.} Finally, we study the performance stability of MOX across various extrapolation levels. Meanwhile, we compare MOX with GA and NPO with different regularization strengths, \textit{i.e.}, weight value of the unlearning term. We set the weight value from 1 to 5, and show the results in Fig. \ref{fig:stability}. We can see that MOX is consistently the best among the other two methods. Moreover, it remains relatively stable under various extrapolation strengths. On the other side, the baseline methods change drastically as the weight changes. Thus, we validate that MOX is stable under various settings and can be reproduced easily.

Further, we study the stability of MOX under various sizes of forget sets, including $1\%$, $5\%$, and $10\%$. The results shown in Fig. \ref{fig:ratio} demonstrate that MOX performs consistently better than both NPO and GA with a large margin, which validates its effectiveness over varied forget sizes.

\paragraph{Computational Efficiency}
To justify the computational efficiency of MOX, we compute the FLOPs of model extrapolation and the computation time using Llama2-7B. We compare the FLOPs of our GD training with NPO, which both include a forward and backward propagation; they both yield 43 TFLOPs per sample. After GD training, we compute the FLOPs of our model extrapolation and yield 21 GFLOPs, which is negligible compared to the previous training cost. For computing 5 $\theta_{mem}$ models with varied extrapolation strengths, it takes 4.7s. When conducting a momentum update for one $\theta_{mem}$ model, it takes 2.5s. Moreover, we note that our MOX can be plugged into any training phase, which won't affect the whole training process. The memory cost is 3$\times \theta_{ref}$, and NPO is 2$\times \theta_{ref}$. But since MOX is efficient to compute, we can move the process to the CPU, which leaves the same computational cost as NPO on the GPU. Therefore, we can validate the efficiency and flexibility of our method.

\section{Conclusion}
In this paper, we study Machine Unlearning without GA to avoid its deficiencies. Instead of obtaining the forget model directly, we propose to first enhance memorization on the forget set. As a result, we obtain a memorization model opposite to our desired model. To achieve unlearning, we conduct model extrapolation that reaches the opposite direction from memorization. Thus, the extrapolation yields a forget model that effectively unlearns the undesirable knowledge. Extensive evaluations justify the effectiveness and stability of our method, which shows superior forgetting performance without sacrificing model utility. Moreover, we propose targeted unlearning and momentum ensembling variants, which further enhance the adaptability and effectiveness of our method. In future studies, we hope to further understand how ensembling achieves superior performance and theoretically understand its generalization superiority.

\bibliographystyle{unsrt}
\bibliography{references}

%%%%%%%%%%%%%%%%%%%%%%%%%%%%%%%%%%%%%%%%%%%%%%%%%%%%%%%%%%%%%%%%%%%%%%%%%%%%%%%
%%%%%%%%%%%%%%%%%%%%%%%%%%%%%%%%%%%%%%%%%%%%%%%%%%%%%%%%%%%%%%%%%%%%%%%%%%%%%%%
% APPENDIX
%%%%%%%%%%%%%%%%%%%%%%%%%%%%%%%%%%%%%%%%%%%%%%%%%%%%%%%%%%%%%%%%%%%%%%%%%%%%%%%
%%%%%%%%%%%%%%%%%%%%%%%%%%%%%%%%%%%%%%%%%%%%%%%%%%%%%%%%%%%%%%%%%%%%%%%%%%%%%%%
\appendix
\clearpage

\begin{center}
 \rule{6.50in}{1.2pt}\\
 {\Large\bf Appendix for ``Is Gradient Ascent Really Necessary?\\ \vspace{0.06in}Memorize to Forget for Machine Unlearning''}
 \rule{6.50in}{1.2pt}
\end{center}

\vskip 0.3in

In this appendix, we first discuss related works. Then, we elucidate our implementation details. Then, we prove Theorem \ref{theorem} from the main paper. Further, we provide unlearning examples of MOX compared to other baseline methods and complementary analyses to further validate the performance superiority of our method MOX. Particularly, we conduct additional empirical analyses to show the unlearning trajectories. Finally, we study the continual unlearning problem where the forget set keeps changing, and compare the performance of MOX with several typical methods.

\section{Related Works}
\label{sec:related_works}
MU was taken seriously under several regulatory catalysts, such as the General Data Protection Regulation (GDPR)~\cite{regulation2018general} and the California Consumer Privacy Act (CCPA)~\cite{bonta2022california}, and was first introduced by Cao \& Yang~\cite{cao2015towards}. Pioneering works mainly focused on small-scale models under traditional tasks~\cite{bourtoule2021machine, golatkar2020eternal, ginart2019making, hsieh2019classification, thudi2022unrolling, izzo2021approximate, koh2017understanding, guo2019certified, sekhari2021remember}. Particularly, certified unlearning~\cite{ginart2019making, bourtoule2021machine} provides provably unlearning strategies under certain scenarios, and influence function studies the influence of removing data from training. However, these studies normally require second-order Hessian matrix computation~\cite{foret2020sharpness,huang2023flatmatch}, which is intractable to modern architectures such as LLMs.

Along with the development of foundation models, most research focusing on unlearning has shifted to LLMs~\cite{jang2022knowledge, wang2023kga, chen2023unlearn, yao2024large, li2024wmdp,yang2025exploring}. Specifically, GA~\cite{maini2024tofu, hsieh2019classification, zhang2024negative, wang2025rethinking, wang2024llm} is an intuitive way to remove the knowledge that has been learned via GD. The studies proposed variants of GA methods to achieve balanced performance between model utility and forget quality. However, most of the studies can be generalized into loss reweighting, which still keeps the deficiencies of GA, such as training instability and catastrophic collapse. Particularly, Jang et al.~\cite{jang2022knowledge} proposed to maximize the loss of next-token prediction, which successfully achieves MU for LLMs. Yao et al.~\cite{yao2024machine} proposed to combine GA with GD on in-distribution data, which alleviates the negative effect of GA to some extent. Further, NPO~\cite{zhang2024negative} proposes to optimize the unlearning model as a weighted preference optimization problem, which only focuses on using the forget set as a negative preference. 

Another branch for unlearning is model editing~\cite{guo2024mechanistic, ilharco2022editing, jung2025come, wu2023depn}, which aims to operate on the model weight to forget the specific knowledge held in it. Wu et al.~\cite{wu2023depn} proposed a Privacy Neuron Detector, which can effectively detect privacy-related neurons to further eliminate their knowledge via a Privacy Editor. However, it is hard to apply the neuron detector to large-scale architectures such as LLMs, and the privacy information is hard to define when the forget set and the retain set are similar to each other. Ilharco et al.~\cite{ilharco2022editing} studied task arithmetic, which demonstrates the relationship between tasks from the perspective of hypothesis vectors. If optimization on one task moves the model towards a certain vector, then the negation of the vector in the hypothesis space would lead to forgetting the task. The study was based on relatively small-scale datasets under classification tasks, and it only considered one task during each experimental trial. However, in MU, there are both forgetting and retaining tasks to be fulfilled. Guo et al.~\cite{guo2024mechanistic} studied mechanism unlearning, which first localizes tokens that are responsible for extracting facts or knowledge, then it modifies the fact prediction via an MLP layer. However, identifying the tokens and deciding the proper facts is difficult in practice, thus limiting its extension to large-scale MU.

Our MOX approach avoids the deficiencies of GA and stabilizes the unlearning process by only conducting GD, thus showing advantages over GA-based methods. Moreover, MOX resembles task arithmetic, but it is more compatible with large-scale unlearning practices by taking both forgetting and retaining into account. It is also free from the identification of certain factors or privacy-related neurons; thus, it is easily extendable to large-scale MU applications.

\section{Additional Details}
Our experiments can be conducted on a single NVIDIA A100/H100 or 4$\times$NVIDIA 4090 GPU. In our experiments, we use a batch size of 32. For fine-tuning pre-trained Phi-1.5B~\cite{abdin2024phi}, we use a learning rate of $2e-5$ for 5 epochs to obtain the original model. For Llama2-7B~\cite{touvron2023llama}, we also fine-tune for 5 epochs, but with a learning rate of $1e-5$. Both all models, we use AdamW as the optimizer for fine-tuning and unlearning. The unlearning process for all methods, including MOX, utilizes the same learning rate as used during fine-tuning pre-trained LLMs. 

\section{Proof of Theorem}
To prove Theorem~\ref{theorem}, we first formulate the stability of GD under the unlearning setting. We denote $D_{sub}=\lbrace x_i \rbrace_{i=1}^n\in D_{pretrain}$ as the forget set that belongs to the subset of the pretraining dataset. The goal of GD is to minimize the empirical loss
\begin{equation}
    L(\theta) = \frac{1}{n}\sum_{i=1}^n \ell(\theta; x_i).
\end{equation}
The GD is conducted in a step by step:
\begin{equation}
    \theta_{t+1} = \theta_t - \eta \nabla L(\theta_t).
\end{equation}
\begin{proof}
(Informal) Let $0 < \eta \le 1/L_s$. Then the GD iterates $\{\theta_t\}$ satisfy:
Monotone decrease: 
\begin{equation}
    L(\theta - \eta \nabla L(\theta)) \le L(\theta) - \eta \|\nabla L(\theta)\|^2 + \frac{L_s \eta^2}{2}\|\nabla L(\theta)\|^2\le L(\theta) - \frac{\eta}{2}\|\nabla L(\theta)\|^2.
\end{equation}
        
Hence, $L(\theta_t)$ decreases and is bounded below by $0$. This prevents divergence of loss and shows gradients square-sum is finite, implying $||\Delta L(\theta_t)||\rightarrow 0$ along subsequences.

By applying the PL condition, we have:
\begin{equation}
    \frac{1}{2}\|\nabla L(\theta)\|^2\ge \mu \big( L(\theta) - L^\star \big),
\end{equation}
As $\|\nabla L(\theta_t)\| \to 0$, it must be that $L(\theta_t) \to L^\star$. Therefore, the training converges to the global minimum loss.

Under the Cross Entropy loss, which decomposes into Entropy and KL Divergence:
\begin{equation}
    L(\theta) = \frac{1}{n} \sum_{i=1}^n \left[ H(q(\cdot|c_i)) + D_{KL}(q(\cdot|c_i) \| p_\theta(\cdot|c_i)) \right].
\end{equation}
The global minimum $L^\star$ occurs if and only if $D_{KL} = 0$ for all samples (assuming the model capacity allows realizability), meaning:
\begin{equation}
    p_{\theta^\star}(\cdot|c_i) = q(\cdot|c_i).
\end{equation}

By the definition of collapse, $p_{\theta^*}$ is nearly constant (low entropy). However, by the Data diversity assumption, $q(\cdot|c_i)$ has non-zero entropy (high diversity). If $p_{\theta^*}$ is collapsed (constant) and $q$ is diverse, the KL divergence (and thus the loss) must be strictly greater than the minimum possible entropy:
\begin{equation}
    L(\theta_{collapsed}) \gg \frac{1}{n} \sum H(q) = L^\star.
\end{equation}

However, we have shown that GD converges to $L^\star$. 
\begin{equation}
    L(\theta^*) = L^\star \implies p_{\theta^*} = q.
\end{equation}
Since $q$ is not collapsed, $p_{\theta^*}$ cannot be collapsed.
\end{proof}

\section{Unlearning Example}
To intuitively understand how our MOX works to unlearn specific knowledge, we conduct experiments on the Harry Potter (HP) dataset~\cite{eldan2023who}. We use Llama2-7B and compare our method, MOX, with the original response. The result is shown in Tab. \ref{tab:example}. We can see that the original response trained to fit the HP dataset contains lots of sensitive information, such as names from the Harry Potter book. However, the response from MOX is more general and free from the sensitive information from the HP dataset. Therefore, we can again testify that MOX can effectively unlearn sensitive knowledge and produce privacy-free content for generalized tasks.

\begin{table}[h]
\centering
\caption{Unlearning examples of the original Llama2-7B and our MOX on generating responses on the HP dataset.}
\begin{tabular}{p{13.5cm}}
\toprule
\textit{\textbf{Prompt}}: Let's see how you would complete this piece of text: The weather did not improve as the day wore on, so that at seven o’clock that evening, when Harry and Ron went down to the Quidditch pitch for practice, they were soaked through within minutes, their feet slipping and sliding on the sodden grass. \\
\\
\textit{\textbf{Original response}}: The rain was pouring, thundery gray. making it difficult to see the Quaffle, and Professor McGonagall had to intervene, calling off the practice early to prevent any accidents. \\
\\
\textit{\textbf{MOX}}: The weather did not improve as the day wore on, and when they arrived at the field, they were soaked through within minutes, their feet slipping and sliding on the sodden grass, every movement a battle against the biting wind and driving rain. \\
\midrule
\textit{\textbf{Question}}: Let's see how you would complete this piece of text: It was the best he could do, for the icy feeling on his arm not holding the cup was not the lingering chill of the water. A slimy white hand had gripped his wrist, and the creature to whom it belonged was pulling him, slowly, backward across the rock. \\
\\
\textit{\textbf{Original response}}: Harry tried to resist, but he was being dragged into the darkness of the lake, the last thing he saw was the horrified faces of Ron and Hermione before he was submerged under the water. \\
\\
\textit{\textbf{MOX}}: It was the best he could do, for the cold feeling on his arm was not the lingering chill of the water. A slimy white hand had gripped his wrist... \\

\bottomrule
\end{tabular}
\label{tab:example}
\end{table}

\section{Unlearning Trajectory}
Further, we conduct experiments on the TOFU~\cite{maini2024tofu} real-world dataset with $5\%$ forget set using Llama2-7B to evaluate the forget quality and model utility performance of various baseline methods, as shown in Fig. \ref{fig:trajectory}. We can observe that several baseline methods, such as GA, GAD, and KL, first have an increase in their forget quality, but suddenly show a significant model utility drop, and further lead to the collapse of the forget quality. As for NPO, we can see that it is more stable compared to the above baselines, but it still shows forget quality degradation at the end of training. Most importantly, all of the baseline methods have serious model utility loss, except the TV method. Because TV also avoids gradient ascent as our MOX does, the training shows no collapse, and the forget quality is increasing for most training stages. Nonetheless, the TV method suffers from limited forget quality performance, which is far from the retained baseline model. Compared to TV, our MOX conducts more effective forgetting by enhancing the extrapolation strength. As a result, MOX can achieve significant performance improvement and still maintain satisfactory model utility results. Moreover, we observe that the momentum updated MOX, \textit{i.e.}, MOX-Mo., can further enhance the unlearning performance on both forget quality and model utility, which again demonstrates the effectiveness of our design choices.

\begin{figure}
    \centering
    \includegraphics[width=0.6\linewidth]{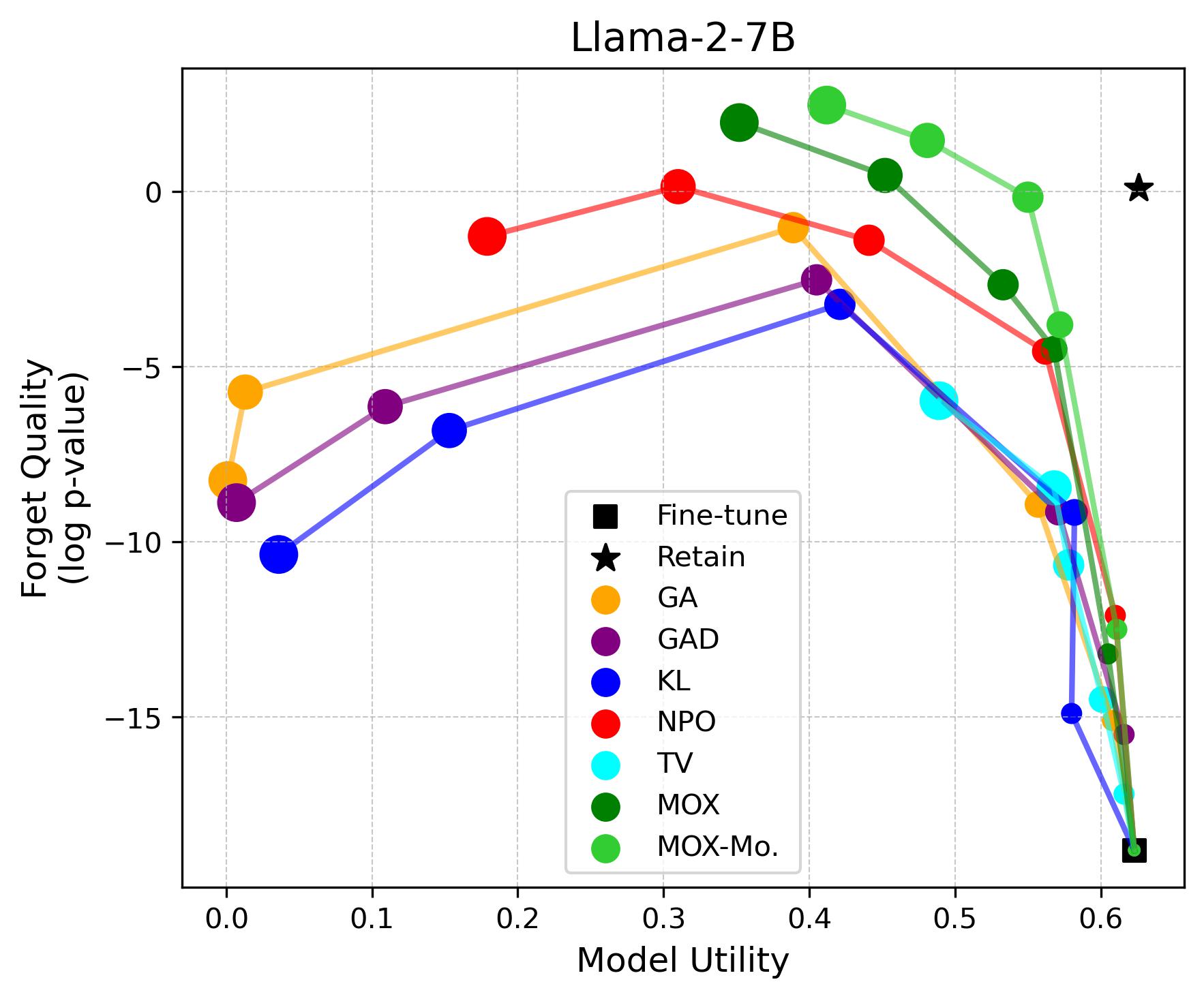}
    \caption{Learning trajectory of various unlearning methods on TOFU dataset with $5\%$ forget set. The relative size of the markers denotes the number of epochs during unlearning.}
    \label{fig:trajectory}
\end{figure}

\begin{figure}[h]
    \centering
    \includegraphics[width=\linewidth]{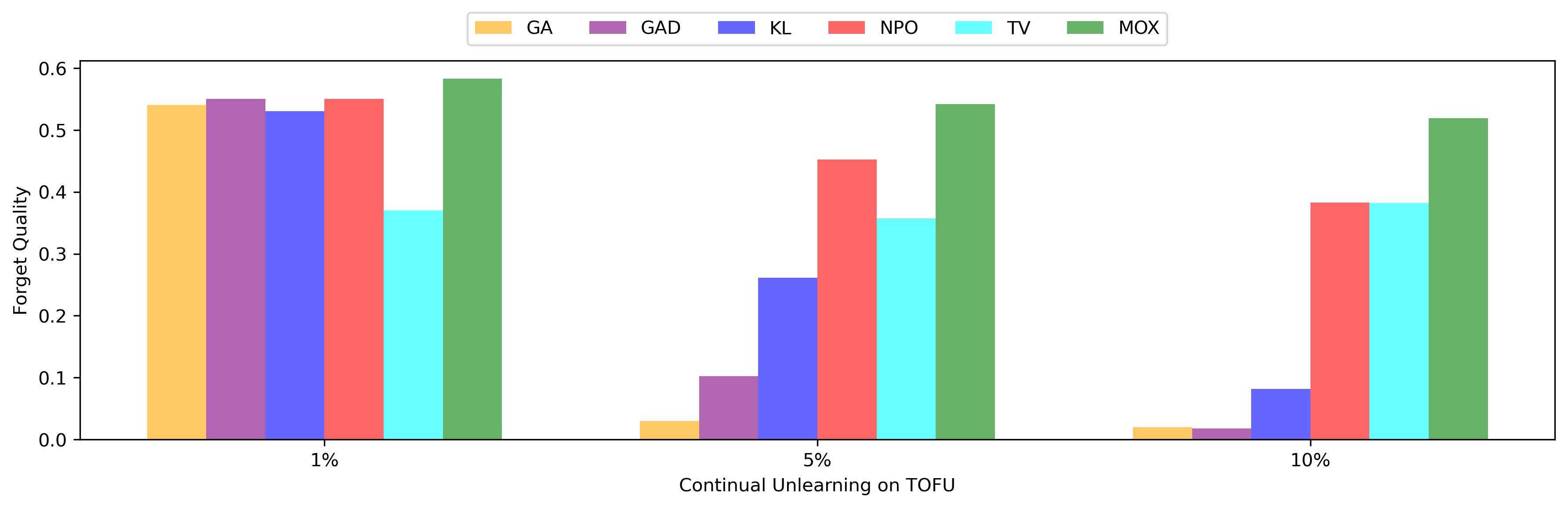}
    \caption{Continual learning performance on TOFU dataset.}
    \label{fig:continual}
\end{figure}

\section{Continual Unlearning}
Moreover, we study a realistic setting where the forgetting dataset keeps changing during training, which introduces significant distirbution shifts~\cite{li2024dynamic,huang2025towards}. We use the TOFU dataset, which contains three different forget sets with varied forget ratios, namely, $1\%$, $5\%$, and $10\%$. We first unlearn using the $1\%$ forget set, then unlearn on the $5\%$ forget set, and further the $10\%$ forget set. For each unlearning process, we fine-tune the model for 3 epochs, and after which we evaluate the forget quality on the real-world set. The result is shown in Fig. \ref{fig:continual}. We can see that some baseline methods, such as GA, GAD, and KL, significantly degrade after training on a different forget set. Among them, we also observe that adding a retaining regularization on the retain set, as GAD and KL do, helps enhance the continual unlearning performance and prevent degradation. Moreover, the TV method is resistant to the change of the unlearning target because it avoids gradient ascent, which could be the reason to cause the instability of continual unlearning. Nonetheless, our MOX shows the most resistant performance against changing forget set and shows consistent performance superiority over all other baseline methods under all settings. Therefore, we can justify the effectiveness of MOX on continual unlearning, which brings a practical advantage for real-world applications.

\begin{table}[h!]
\centering
\begin{tabular}{lcccc}
\toprule
\textbf{Metric} & \textbf{FQ (↑)} & \textbf{MU (↑)} & \textbf{F-RL (↓)} & \textbf{R-RL (↑)} \\
\midrule
GA  & 0.0137 & 0.5745 & 0.4856 & 0.8795 \\
PO  & 0.0501 & 0.6232 & 0.4620 & 0.8755 \\
MOX & \textbf{0.0611} & \textbf{0.6488} & \textbf{0.4500} & \textbf{0.9508} \\
\toprule
\end{tabular}
\caption{Performance comparison under extreme semantic overlap between forget set and retain set.}
\label{tab:semantic_overlap}
\end{table}

\section{Unlearning with Semantic Overlap}
To further validate the performance of MOX when the forgetting set and retaining set are extremely overlapped, \textit{i.e.,} share the same knowledge, we conduct experiments under TOFU by splitting the retain set into a 10\% forget set and the remaining 90\% as the retain set. The result is shown in Tab. \ref{tab:semantic_overlap}. As we can see, our method is effective under strong overlapped retain and forget sets, which justifies the retention performance of MOX.

\begin{table}[h!]
\centering
\begin{tabular}{ll}
\toprule
\textbf{Model} & \textbf{WMDP (↓)} \\
\midrule
Original    & 5.21 \\
GA unlearn  & 1.53 \\
GA relearn  & 4.88 (+3.35) \\
NPO unlearn & 0.98 \\
NPO relearn & 5.01 (+4.03) \\
MOX unlearn & \textbf{0.54} \\
MOX relearn & \textbf{3.82} (\textbf{+3.28}) \\
\toprule
\end{tabular}
\caption{Relearning performance comparison.}
\label{tab:relearning}
\end{table}

\section{Analysis on Relearning Attack}
Relearning~\cite{hu2024unlearning} is a popular way to recover the unlearning progress, which might be detrimental to our target. Thus, to evaluate the performance of MOX under relearning attacks, we conduct experiments to verify the robustness of MOX. We consider the WMDP benchmark using Llama2-7B, and follow the same setting as Hu et al.~\cite{hu2024unlearning} to show the relearning performance of MOX in Tab. \ref{tab:relearning}. We find out that the relearning performance of MOX is much more resistant than other baselines, this is because extrapolation identifies the difference between the current model and the original model, and further enhances such a difference to achieve forgetting. Hence, relearning can easily recover the effect of existing approaches, but MOX can still be resistant to the recovery during relearning.

%%% Uncomment this section and comment out the \bibliography{references} line above to use inline references.
% \begin{thebibliography}{1}

% 	\bibitem{kour2014real}
% 	George Kour and Raid Saabne.
% 	\newblock Real-time segmentation of on-line handwritten arabic script.
% 	\newblock In {\em Frontiers in Handwriting Recognition (ICFHR), 2014 14th
% 			International Conference on}, pages 417--422. IEEE, 2014.

% 	\bibitem{kour2014fast}
% 	George Kour and Raid Saabne.
% 	\newblock Fast classification of handwritten on-line arabic characters.
% 	\newblock In {\em Soft Computing and Pattern Recognition (SoCPaR), 2014 6th
% 			International Conference of}, pages 312--318. IEEE, 2014.

% 	\bibitem{hadash2018estimate}
% 	Guy Hadash, Einat Kermany, Boaz Carmeli, Ofer Lavi, George Kour, and Alon
% 	Jacovi.
% 	\newblock Estimate and replace: A novel approach to integrating deep neural
% 	networks with existing applications.
% 	\newblock {\em arXiv preprint arXiv:1804.09028}, 2018.

% \end{thebibliography}

\end{document}